\documentclass{article}

\usepackage{microtype}
\usepackage{graphicx}
\usepackage{subfigure}
\usepackage{booktabs} 
\usepackage[algoruled,vlined]{algorithm2e}
\usepackage[utf8]{inputenc} 
\usepackage[T1]{fontenc}    
\usepackage{hyperref}       
\usepackage{url}            
\usepackage{amsfonts}       
\usepackage{nicefrac}       
\usepackage{xcolor}         
\usepackage{multirow}
\usepackage{multicol}
\usepackage{tabularx}
\usepackage{breqn}
\usepackage{wrapfig}
\usepackage{float}
\usepackage{caption}
\usepackage{subcaption}

\newcommand{\dtsn}{{ Differentiable Tree Search Network }} 

\usepackage[accepted]{icml2024}

\usepackage{amsmath}
\usepackage{amssymb}
\usepackage{mathtools}
\usepackage{amsthm}

\usepackage[capitalize,noabbrev]{cleveref}

\theoremstyle{plain}
\newtheorem{theorem}{Theorem}[section]

\newtheorem{lemma}[theorem]{Lemma}

\theoremstyle{definition}

\theoremstyle{remark}

\usepackage[textsize=tiny]{todonotes}

\icmltitlerunning{\dtsn}

\begin{document}

\twocolumn[
    \icmltitle{Differentiable Tree Search Network}
        
    \icmlsetsymbol{equal}{*}

    \begin{icmlauthorlist}
        \icmlauthor{Dixant Mittal}{nus}
        \icmlauthor{Wee Sun Lee}{nus}
    \end{icmlauthorlist}

    \icmlaffiliation{nus}{School of Computing, National University of Singapore, Singapore}

    \icmlcorrespondingauthor{Dixant Mittal}{dixant@comp.nus.edu.sg}


    \vskip 0.3in
]

\begin{abstract}
In decision-making problems with limited training data, policy functions approximated using deep neural networks often exhibit suboptimal performance. An alternative approach involves learning a world model from the limited data and determining actions through online search. However, the performance is adversely affected by compounding errors arising from inaccuracies in the learned world model. While methods like TreeQN have attempted to address these inaccuracies by incorporating algorithmic inductive biases into the neural network architectures, the biases they introduce are often weak and insufficient for complex decision-making tasks. In this work, we introduce \textit{\dtsn} (D-TSN), a novel neural network architecture that significantly strengthens the inductive bias by embedding the algorithmic structure of a best-first online search algorithm. D-TSN employs a learned world model to conduct a fully differentiable online search. The world model is jointly optimized with the search algorithm, enabling the learning of a robust world model and mitigating the effect of prediction inaccuracies. Further, we note that a naive incorporation of best-first search could lead to a discontinuous loss function in the parameter space. We address this issue by adopting a stochastic tree expansion policy, formulating search tree expansion as another decision-making task, and introducing an effective variance reduction technique for the gradient computation. We evaluate D-TSN in an offline-RL setting with a limited training data scenario on Procgen games and grid navigation task, and demonstrate that D-TSN outperforms popular model-free and model-based baselines.
\end{abstract}
\section{Introduction}
Model-free Deep Reinforcement Learning (DRL) has advanced significantly in addressing complex sequential decision-making problems, largely due to advances in Deep Neural Networks (DNN)~\citep{silver2017mastering,berner2019dota,vinyals2019grandmaster}. The networks enable an agent to learn a direct mapping from observations to actions through a policy~\citep{mnih2016asynchronous,schulman2015trust,schulman2017proximal,cobbe2020ppg} or Q-value function~\citep{mnih2013playing,WangSHHLF16,HasseltGS16,hessel2018rainbow,badia2020agent57}. However, DRL has high sample complexity and limited generalization, limiting its wider application in complex real-world settings, especially when a limited training data is available.

Model-based Reinforcement Learning (MBRL) approaches~\citep{kaiser2019model,hafner2019dream,ha2018recurrent} address these issues by learning world models from limited environment interactions and subsequently using these models for online search~\citep{hafner2019dream,hafner2019learning} or policy learning~\citep{kaiser2019model}. Although MBRL is effective for problems where learning the world model is simpler than learning the policy, its efficacy diminishes for complex, long-horizon problems due to accumulated errors arising from inevitable inaccuracies in the learned world model.

Some recent attempts~\citep{tamar2016value,silver2017predictron,lee2018gated,farquhar2018treeqn} have tried to improve the sample efficiency and generalization capabilities of DNN models by incorporating algorithmic inductive biases into the neural network structure. These inductive biases restrict the class of functions a model can learn by leveraging domain-specific knowledge, reducing the likelihood of the model overfitting to the training data and learning a model incompatible with the domain knowledge. The algorithmic inductive biases in these works are either based on the value iteration algorithm, which has difficulties scaling to very large state spaces, or do not incorporate search or optimization within the architecture, hence reducing their effectiveness.

A notable recent work, TreeQN~\citep{farquhar2018treeqn}, combines look-ahead tree search with deep neural networks. It dynamically constructs a computation graph by fully expanding the search tree up to a predefined depth $d$ using a learned world model and computing the Q-values at the root node by recursively applying Bellman equation on the tree nodes. The whole structure is trained end-to-end and enables learning a robust world model that is helpful in the look-ahead tree search. Consequently, TreeQN outperforms conventional neural network architectures on multiple Atari games. However, the size of the full search tree grows exponentially in the depth, which computationally limits the TreeQN to perform only a shallow search. Consequently, TreeQN is unable to fully exploit the domain knowledge of advanced online search algorithms that can handle problems requiring much deeper search.

In this paper, we first show that the loss function applied on Q-function represented by TreeQN is continuous in the network's parameter space, contributing to its success when used with gradient based learning methods. However, extending TreeQN to allow a more sophisticated search algorithm that learns to adapt the search tree during the search process faces multiple hurdles. First, a naive incorporation of search algorithm in the network architecture might lead to a loss function which is discontinuous in the algorithm's parameter space. This is because a small change in parameters could result in construction of a different search tree and bring large change in the output Q-value and the corresponding loss function. Consequently, the loss function could become difficult to optimize using gradient-based optimization techniques as they generally require a continuous function for optimization. To alleviate this issue, we propose employing a stochastic tree expansion policy and optimizing the expected loss, ensuring the continuity of loss function in the algorithm's parameter space. Furthermore, we propose formulating the search tree expansion as another decision-making task with the goal of progressively minimizing the prediction error in the Q-value, and refining the tree expansion policy via REINFORCE~\citep{williams1992simple}. The use of the REINFORCE algorithm introduces another hurdle; REINFORCE has high variance in its gradient estimate. To handle that, we propose the use of a baseline for variance reduction, adopting the telescoping sum trick used in \citet{guez2018learning} for constructing the baseline.

Employing these techniques allows us to overcome the potential discontinuity issues when an adaptive search algorithm is used, resulting in \dtsn (D-TSN), a novel neural network architecture that embeds the structure of a best-first online search algorithm into the network architecture. D-TSN has a modular design and consists several learnable submodules that dynamically combine to construct a computation graph by following a best-first search algorithm. Further, it employs a learned world model to conduct an end-to-end differentiable search, which allows joint optimization of search submodules and the learned world model using gradient-based optimization. Joint optimization enables learning a robust world model that is useful in the online search and optimizes search submodules to account for the inaccuracies in the learned world model.

We compare D-TSN against popular baselines that include model-free Q-network, to assess the impact of incorporating algorithmic inductive biases into the network architecture; TreeQN, to assess the benefits of incorporating a better online search algorithm capable of performing a deeper search within the same computational budget; and model-based search, to evaluate the advantages of joint optimization of the world model with the search algorithm. We evaluate these methods on deterministic decision-making problems, including Procgen games~\citep{cobbe2020leveraging} and a small-scale 2D grid navigation problem, under limited training data framework to examine their sample complexity and generalization capabilities. Empirical evaluations demonstrate that D-TSN outperforms the baselines in both Procgen and navigation tasks.
\section{Related Works}
In recent years, Reinforcement Learning (RL) has undergone remarkable advancements, notably due to the integration of Deep Neural Networks (DNNs) in domains such as Q-learning~\citep{mnih2013playing,WangSHHLF16,HasseltGS16,hessel2018rainbow,espeholt2018impala,badia2020agent57} and policy learning~\citep{williams1992simple,schulman2015trust,schulman2017proximal,mnih2016asynchronous}. Moreover, model-based RL has been significantly enriched by a series of innovative works~\citep{deisenroth2011pilco,nagabandi2017neural,chua2018deep} that learn World Models~\citep{ha2018recurrent} from pixels and subsequently utilize them for planning~\citep{hafner2019dream,hafner2019learning,schrittwieser2020mastering} and policy learning~\citep{kaiser2019model}.

An intriguing dimension of research in RL seeks to merge learning and planning paradigms. AlphaZero~\citep{silver2017mastering} utilizes DNNs as heuristics with Monte Carlo Tree Search (MCTS). On the other hand, use of inductive biases~\citep{hessel2019inductive} has been explored for their impact on policy learning. Recent works like Value Iteration Network (VIN)~\citep{tamar2016value} represents value iteration algorithm in gridworld domains using a convolutional neural network. Similarly, Gated Path Planning Network~\citep{lee2018gated} replaces convolution blocks in VIN with LSTM blocks to address vanishing or exploding gradient issue. Further, Neural A* Search~\citep{YonetaniTBNK21} embeds A* algorithm into network architecture and learns a cost function from a gridworld map. Alternatively, Neural Admissible Relaxation (NEAR)~\citep{ShahZSVYC20} learns an approximately admissible heuristics for A* algorithm. Works like Predictron~\citep{silver2017predictron} and Value Prediction Network~\citep{oh2017value} employs a learned world model to simulate rollouts and accumulate internal rewards in an end-to-end framework. Extending this paradigm, Imagination-based Planner (IBP)~\citep{pascanu2017learning} utilizes an unstructured memory representation to collate information from the internal rollouts. Alternatively, MCTSnets~\citep{guez2018learning} embeds the structure of Monte Carlo Tree Search (MCTS) in the network architecture and steers the search using parameterised memory embeddings stored in a tree structure. However, \citet{guez2019investigation} suggests that recurrent neural networks could display certain planning properties without requiring a specific algorithmic structure in network architecture.

Our work draws inspiration from these recent contributions, notably TreeQN~\citep{farquhar2018treeqn}. TreeQN employs a full tree expansion up to a fixed depth, followed by recursive updates to the value estimate of tree nodes in order to predict the Q-value at the input state. However, TreeQN's full tree expansion is exponential in the depth, rendering it computationally expensive for tackling complex planning problems that necessitate deeper search. Our work addresses this issue by incorporating a more advanced online search algorithm that emphasizes expansion in promising areas of the search tree.
\section{\dtsn}\label{section_method}
\dtsn (D-TSN) is a neural network design that incorporates the algorithmic inductive bias of a best-first search algorithm into the network structure and employs a learned world model to perform the online search. The learned world model is jointly optimized with the search algorithm, so that the learned world model, although imperfect, is useful for the online search, and search submodules are robust against errors in the world model.

\subsection{Learnable Submodules}\label{section_method_learnable_submodules}
D-TSN comprises several learnable submodules that function as subroutines in a best-first search algorithm and dynamically construct a computation graph. An illustration is provided in \cref{figure_illustrations_submodules}.

The \textit{Encoder} module ($\mathcal{E}_\theta$) transforms the actual state $s_t$ into a latent state $h_t = \mathcal{E}_\theta(s_t)$, facilitating online search within a latent space. The \textit{Transition} module ($\mathcal{T}_\theta$) approximates the environment's transition function, using $h_t$ and action $a_t$ to predict the subsequent latent state, $h_{t+1} = \mathcal{T}_\theta (h_t,a_t)$, of the transition. The \textit{Reward} module ($\mathcal{R}_\theta$) approximates the environment's reward function, predicting the reward, $r_t = \mathcal{R}_\theta(h_t,a_t)$, for the transition based on $h_t$ and $a_t$. The \textit{Value} module ($\mathcal{V}_\theta$) maps latent state $h_t$ to its estimated state value, $\mathcal{V}_\theta(h_t)$.

Dividing the network into these submodules reduces total learnable parameters and injects significant planning bias into the network architecture, preventing overfitting to an arbitrary function that may align with the limited available training data.

\subsection{Tree Search in Latent Space} \label{section_method_tree_search}
The search begins by encoding the input state $s_0$ into its latent state $h_0$. It proceeds through expansion and backup phases. In the expansion phase, the search tree expands iteratively for a predetermined number of expansion steps. During the backup phase, Q-values at the root node are recursively computed using the Bellman equation across expanded tree nodes. Each tree node represents a latent state reachable from the root, while branches represent actions taken at the nodes. A set of candidate nodes, $O$, is maintained during the expansion phase, representing the tree nodes eligible for further expansion. The D-TSN algorithm is detailed in \cref{algorithm_dtsn} in the appendix.

\subsubsection{Expansion Phase}\label{section_method_tree_search_expansion}
Each search iteration begins by evaluating the path value, $\Bar{V}(N)$, of the candidate nodes. The path value is defined as the cumulative sum of rewards from the root node to a particular leaf node $N$, in addition to the value of the leaf node predicted by the value module ($\mathcal{V}_\theta$), i.e. 
\begin{equation} \label{equation_path_value_of_node}
    \Bar{V}(N) = \mathcal{R}_\theta(h_0,a_0) + ... + \mathcal{V}_\theta(h_N) 
\end{equation}
A naive implementation of the search selects the node $\hat{N}$ with the highest total path value for expansion; however, for differentiable search, the node $\hat{N}$ is sampled from the candidates using a distribution constructed by applying softmax over the path values of the candidates. Expansion of node $\hat{N}$ is carried out by simulating every action on the node using the transition module ($\mathcal{T}_\theta$). Simultaneously, the associated reward, $\mathcal{R}_\theta(h_{\hat{N}},a)$, is also computed. The resulting latent states are added to the tree as children of node $\hat{N}$. Additionally, they are added to the candidate set $O$ for subsequent expansions, while $\hat{N}$ is excluded from the set. This can be represented as:
\begin{equation*}
    O \leftarrow O \cup \{h_{a} |\ h_{a} = \mathcal{T}_\theta(h_{\hat{N}},a); \forall a \in A\} - \hat{N}
\end{equation*}

\subsubsection{Backup Phase} \label{section_method_tree_search_backup}
The expansion phase is followed by the \textit{Backup phase}. In this phase, value of all tree nodes is recursively updated using the Bellman equation as follows:
\begin{align*}
    Q(N,a) &=  \mathcal{R}_\theta(h_N,a) + V( N')
\end{align*}
\begin{align*}
    \text{where, \quad } N' &= \mathcal{T}_\theta(h_N,a) \\
    V(N') &= \begin{cases}
                \mathcal{V}_\theta(h_{N'}),              & \text{N' is a leaf node} \\
                \max_a Q(N',a),         & \text{otherwise}
            \end{cases} 
\end{align*}
After the backup phase, the Q-values at the root node are returned as the final output of the online search.

\subsection{Construction of Computation Graph}
Throughout the expansion and backup phases, illustrated in \cref{figure_illustrations_expansion,,figure_illustrations_backup} respectively, a dynamic computation graph is constructed where the output Q-values depend on the combination of all the submodules, i.e. Encoder, Transition, Reward, and Value modules. During training, the output Q-function is evaluated by a loss function and is optimized using gradient-based optimizers such as Stochastic Gradient Descent (SGD). These optimizers backpropagate the gradient of this loss through the entire computation graph, and update the parameters of submodules in an end-to-end process. An illustration of the process is provided in \cref{figure_illustrations_graph}.

\subsection{Discontinuity of the Loss Function}
A pivotal aspect of optimizing D-TSN's parameters through gradient descent is ensuring that the loss function applied to the output Q-function is \textit{continuous in the network's parameter space}. We initiate this discussion by showing that the loss function applied to the Q-function represented by TreeQN is continuous in the parameter space. 

\begin{theorem}\label{theorem_continuity_of_q_function}
Given a set of parameterised modules that are continuous in the parameter space $\theta$, the Q-function computed by fully expanding a search tree to a fixed depth `$d$' by composing these modules and backpropagating the children values using addition and max operations is continuous in the parameter space $\theta$. (See \cref{appendix-section-continuity} for the proof.)
\end{theorem}

\cref{theorem_continuity_of_q_function} establishes that the Q-function calculated by a full tree expansion, as performed in TreeQN, is continuous in the network's parameter space. Consequently, the loss function applied to this Q-function is also continuous, contributing to TreeQN's success in gradient-based optimization. 

On the other hand, a naive implementation of D-TSN, as outlined in \cref{section_method_tree_search_expansion}, approximates full tree expansion by expanding only those paths likely to represent the optimal trajectory from the root node. This results in the output Q-function and the corresponding loss function being dependent on the specific tree, $\tau$, constructed during the search:
\begin{equation}\label{equation_naive_loss}
    L(s,a) = \mathcal{L}(Q_\theta(s,a|\tau))
\end{equation}
When the network parameters are changed slightly, the naive implementation could generate a different tree structure, which, in turn, would induce large change in the loss function and affect its continuity in the parameter space.

\subsection{Enabling Continuity in the Loss Function}
To overcome the discontinuity issue observed in the naive implementation of D-TSN, we employ a stochastic tree expansion policy. This approach allows us to optimize the expectation of the loss function, defined as:
\begin{align}\label{equation_expected_loss}
    L(s,a) 
    &= \mathbb{E}_{\tau} \Big[\mathcal{L}(Q_\theta(s,a|\tau))\Big] \\
    &= \sum_\tau \pi_\theta(\tau) \mathcal{L}(Q_\theta(s,a|\tau)) \nonumber
\end{align}
The expected loss in \cref{equation_expected_loss} is continuous in the parameter space $\theta$ and can be optimized using well-known gradient-based optimization techniques.

\subsection{Stochastic Tree Expansion Policy}
In order to compute the expected loss in \cref{equation_expected_loss}, let us represent a partial search tree after $t$ iterations as $\tau_t$. The output Q-values, denoted as $Q_\theta(s,a | \tau)$, depends on the final tree $\tau$ sampled after $T$ iterations. We can define a stochastic tree expansion policy $\pi_\theta(\tau_t)$ that takes a tree $\tau_t$ as input and outputs a distribution over the candidate nodes, facilitating stochastic selection of the node for further expansion and generating the tree $\tau_{t+1}$. We compute the stochastic tree expansion policy by taking softmax over the path value, as defined in \cref{equation_path_value_of_node}, of each candidate node as follows:
\begin{equation}
    \pi_\theta(n|\tau_t) = \dfrac{exp(\Bar{V}(n))}{\sum_{n' \in O (\tau_t)} \ exp(\Bar{V}(n'))}
\end{equation}
 The gradient of the expected loss~\citep{SchulmanHWA15} in \cref{equation_expected_loss} can be computed as follows (See \cref{appendix_section_derivation_of_gradient} for the derivation):
\begin{dmath} \label{equation_tree_policy_gradient}
    \nabla_\theta L(s,a) = \mathbb{E}_{\tau} \Big[ \mathcal{L}(Q_\theta(s,a|\tau)) \sum_{t=1}^T \nabla_\theta \log \pi_\theta(n_t|\tau_t) 
    + \nabla_\theta \mathcal{L}(Q_\theta(s,a|\tau)) \Big]
\end{dmath}

\subsection{Reducing Variance using Telescopic Sum} \label{section_method_telescoping_sum}
The REINFORCE term of the gradient in \cref{equation_tree_policy_gradient}
usually has high variance due to the difficulty of credit assignment~\citep{guez2018learning} in a reinforcement learning type objective; the second part of the gradient equation is the usual optimization of a loss function, so we expect it to be reasonably well behaved. To reduce the variance of the first part of the gradient, we take inspiration from the telescoping sum trick in~\citet{guez2018learning}.

Let us denote the loss after $t$ iterations as $L_t= \mathcal{L}(Q_\theta(s,a|\tau_{t}))$. The objective is to minimize (or equivalently, maximize the negative of) the loss after $T$ iterations, represented as $L_T$. Assuming that $L_0=0$, we can rewrite $L_T$ as a telescoping sum:
\begin{equation*}
    L_T = L_T - L_0 = \sum_{t=1}^T L_t - L_{t-1}
\end{equation*}
Now, we define a reward term, $r_t$, for selecting node $n$ during the $t^{th}$ iteration as the reduction in the loss value after the $t^{th}$ iteration, i.e. $r_t = L_t - L_{t-1}$. Further, let us represent the return (or the sum of rewards) from iteration $t$ to the final iteration $T$ as $R_t$, which can be computed as:
\begin{equation*}
    R_t = \sum_{i=t}^T r_i = L_T - L_{t-1}
\end{equation*}
Given this, the REINFORCE term from \cref{equation_tree_policy_gradient} can be reformulated as $\sum_t^T \nabla_\theta \log \pi_\theta(n_t|\tau_t) R_t$, where $L_{t-1}$ serves as a baseline that helps in reducing variance. Consequently, the final gradient estimate of the loss in \cref{equation_expected_loss} is expressed as:
\begin{dmath}\label{equation_tree_policy_gradient_telescopic}
    \nabla_\theta L(s,a) = \mathbb{E}_{\tau} \Big[\sum_t^T \nabla_\theta \log \pi_\theta(n_t|\tau_t) R_t
    + \nabla_\theta \mathcal{L}(Q_\theta(s,a|\tau)) \Big] 
\end{dmath} 
For empirical evaluations, we use a single sample estimate~\cite{SchulmanHWA15} of the expected gradient in \cref{equation_tree_policy_gradient_telescopic}.
\section{Experiments}
\subsection{Test Domains}
\subsubsection{Navigation:} 
This is a 2D grid-based navigation task designed to quantitatively and qualitatively visualize the agent's generalization capabilities. The environment is a $20\times20$ grid featuring a central hall. At the start of each episode, a robot is randomly positioned inside this hall, while its destination is set outside. We present two distinct scenarios: one with a single exit and another with two exits from the central hall. Training is done only in the two exit scenario. The single exit scenario, which requires a longer-horizon planning to reach the goal, is used to test generalization.

\subsubsection{Procgen:} 
Procgen is a collection of 16 procedurally generated, game-like environments, specifically designed to evaluate an agent's generalization capability, differentiating it from Atari 2600 games~\citep{mnih2013playing}. The open-source code for these environments can be found at \url{https://github.com/openai/procgen}. Further details on these domains are presented in \cref{appendix_experiments_domains}.

\subsection{Learning Framework}
D-TSN can serve as a drop-in replacement for conventional Convolutional Neural Network (CNN) architectures. It can be trained using both online and offline reinforcement learning algorithms. In this paper, we employ the offline reinforcement learning (Offline-RL) framework to focus on the sample complexity and generalization capabilities of D-TSN when compared with the baselines. Offline-RL, often referred to as batch-RL, is the scenario wherein an agent learns its policy solely from a fixed dataset of experiences, without further interactions with the environment. Please refer to \cref{appendix_experiments_learning_setup} for a detailed explanation of the learning framework.

\subsubsection{Training Datasets}
We use a behavior policy, which can be \textit{optimal or sub-optimal}, to collect the offline training dataset. An optimal policy generates a dataset with lower noise and a cleaner training signal, leading to a stable learning process. In contrast, a sub-optimal policy produces a noisier dataset, which consequently restricts the quality of the policy that the agent can learn. The choice of the behavior policy depends on domain-specific requirements and the resources available for data collection. Our evaluations explore datasets generated using both optimal (in the Navigation domain) and sub-optimal (in the Procgen domain) behavior policies.

The training dataset, $\mathcal{D}$, consists of trajectories generated using the behavior policy $\pi_B$, where each trajectory, $\tau_i$, is defined as a series of $T$ tuples, each comprising the state observed, action taken, reward observed, and the corresponding Q-value of the observed state, denoted as $\tau_i = \Big\{(s_{t,i},\ a_{t,i},\ r_{t,i},\ Q_{t,i})\Big\}_{t=0}^{T}$. The Q-value for state $s_{t,i}$ can be computed by adding the rewards obtained in the trajectory from timestep $t$ onwards, i.e. $Q_{t,i} = \sum_{j=t}^T {r_{j,i}}$. We limit the number of trajectories to $1000$ for each domain to evaluate the sample complexity and generalization capabilities of each method.

\subsubsection{Loss Function}
The primary objective is to make the Q-values, as computed by a method, closely approximate the observed Q-values for corresponding states and actions. To achieve this, we minimize the mean squared error between the predicted and observed Q-values. This loss, denoted as $\mathcal{L}_{Q}$, is defined as:
\begin{equation} \label{equation_mse_loss}
\mathcal{L}_{Q} = \mathbb{E}_{(s,a,Q) \sim \mathcal{D}} (
Q_\theta(s,\ a) - Q )^2
\end{equation}
However, in the offline-RL setting, there is a risk that Q-values for out-of-distribution actions could be overestimated. To address this, we incorporate the CQL~\citep{kumar2020conservative} loss, which encourages the agent to adhere to actions observed within the training data distribution. This loss, $\mathcal{L}_{\mathcal{D}}$, is defined as:
\begin{equation} \label{equation_cql_loss}
\mathcal{L}_{\mathcal{D}} = \mathbb{E}_{(s,a) \sim \mathcal{D}} \Big( \log \sum_{a'} \exp ( Q_\theta(s,\ a') ) - Q_\theta(s,\ a) \Big)
\end{equation}
Additionally, as the search in D-TSN is performed in the latent space, we incorporate self-supervised consistency loss functions~\citep{schwarzer2020data,ye2021mastering} to ensure consistency in the transition and reward networks. Consider actual states $s_{t}$ and $s_{t+1}$, where $s_{t+1}$ is obtained by taking action $a_{t}$ in state $s_{t}$. Their corresponding latent state representations are denoted as $h_{t}$ and $h_{t+1}$. Here, $h_{t} = \mathcal{E}_\theta(s_{t})$ and $h_{t+1} = \mathcal{E}_\theta(s_{t+1})$. Now, we can use the transition module to predict another latent representation of state $s_{t+1}$, represented as $\Bar{h}_{t+1} = \mathcal{T}_\theta(h_{t},\ a_{t})$. To ensure that the transition function, $\mathcal{T}_\theta$, provides consistent predictions for the transitions in the latent space, we minimize the squared error between the latent representations $h_{t+1}$ and $\Bar{h}_{t+1}$.
\begin{equation} \label{equation_transition_loss}
\mathcal{L}_{\mathcal{T}_\theta} =\mathbb{E}_{(s_t,\ a_t,\ s_{t+1}) \sim \mathcal{D}} ( \Bar{h}_{t+1} - h_{t+1} )^2
\end{equation}
In a similar vein, we minimize the mean squared error between the predicted reward $\mathcal{R}_\theta(h_{t}, a_{t})$ and the actual reward observed $r_{t}$ in the training dataset $\mathcal{D}$.
\begin{equation} \label{equation_reward_loss}
\mathcal{L}_{\mathcal{R}_\theta} = \mathbb{E}_{(s_{t},\ a_{t},\ r_{t}) \sim \mathcal{D}} ( \mathcal{R}_\theta(h_{t}, a_{t}) - r_{t} )^2
\end{equation}
Based on each method's specifications, we combine the loss functions from \cref{equation_mse_loss,,equation_cql_loss,,equation_transition_loss,,equation_reward_loss} during training, as detailed in \cref{experiments_baselines}. Additional details on these loss functions are presented in \cref{appendix_experiments_loss_functions}.

\begin{table}
    \caption{Comparison of D-TSN with the baselines on Navigation using metrics \textit{Success Rate} and \textit{Collision Rate}.}
    \label{table_results_on_navigation}
    \centering
    \resizebox{0.48\textwidth}{!}{
        \begin{tabular}{l|cc}
            \toprule
            \multicolumn{1}{c|}{\multirow{2}{*}{Solver}} & Success                                   & Absolute       \\
                                                         & Rate                                      & Collision Rate \\
            \midrule
            \multicolumn{1}{c}{}                         & \multicolumn{2}{c}{\small{\textit{Navigation (2 exits)}}}                   \\
            \midrule
            Model-free Q-network                         & 94.5\% \scriptsize ($\pm$ 0.2\%)          & 4.4\%          \\
            Model-based Search                           & 93.2\% \scriptsize ($\pm$ 0.3\%)          & 6.7\%          \\
            TreeQN                                       & 95.4\% \scriptsize ($\pm$ 0.2\%)          & 3.8\%          \\
            D-TSN                                        & \textbf{99.0\%} \scriptsize ($\pm$ 0.1\%) & \textbf{0.7\%} \\
            \midrule
            \multicolumn{1}{c}{}                         & \multicolumn{2}{c}{\small{\textit{Navigation (1 exit)}}}                    \\
            \midrule
            Model-free Q-network                         & 47.1\% \scriptsize ($\pm$ 0.5\%)          & 50.2\%         \\
            Model-based Search                           & 86.9\% \scriptsize ($\pm$ 0.3\%)          & 12.4\%         \\
            TreeQN                                       & 51.8\% \scriptsize ($\pm$ 0.5\%)          & 39.2\%         \\
            D-TSN                                        & \textbf{99.3}\% \scriptsize ($\pm$ 0.1\%) & \textbf{0.2}\% \\
            \bottomrule
        \end{tabular}
    }
\end{table}

\subsection{Baselines and Implementation Details} \label{experiments_baselines}
We benchmark D-TSN against the following prominent baselines:
    
\textbf{Model-free Q-network:}
This allows us to assess the significance of incorporating inductive biases into the neural network architecture. This model is trained using the loss defined as: 
\begin{equation*}
    \mathcal{L}_{\text{Q-net}} = \lambda_1 \mathcal{L}_{Q} + \lambda_2 \mathcal{L}_{\mathcal{D}}
\end{equation*}    
\textbf{Model-based Search:}
In this baseline, we utilize the submodules defined for D-TSN, but the world models and the value module are trained independently of each other. During evaluation, we employ the best-first search, akin to D-TSN, utilizing the independently trained modules. Through this baseline, we assess the benefits derived from the joint optimization of the world model and the search algorithm. For our evaluations, we perform 10 search iterations for each input state. To train this model, we compute the Q-value, $Q_\theta$, \textit{without performing the search} and optimize the loss defined as: 
\begin{equation*}
    \mathcal{L}_{\text{Search}} = \lambda_1 \mathcal{L}_{Q} + \lambda_2 \mathcal{L}_{\mathcal{D}} + \lambda_3 \mathcal{L}_{\mathcal{T}_\theta} + \lambda_4 \mathcal{L}_{\mathcal{R}_\theta}
\end{equation*}
    
\textbf{TreeQN:}
This comparison helps in highlighting the advantages of an advanced search algorithm, used in D-TSN, that can execute a deeper search while maintaining similar computational constraints. For evaluations, we adhere to a depth of 2, as described in~\citet{farquhar2018treeqn}, for both Procgen and navigation domains. Notably, greater depths, such as 3 or more, are infeasible since the resulting computation graph exceeds the memory capacity (roughly 11GB) of a typical consumer-grade GPU. This model is trained using the loss (as discussed in~\citet{farquhar2018treeqn}) defined as:
\begin{equation*}
    \mathcal{L}_{\text{TreeQN}} = \lambda_1 \mathcal{L}_{Q} + \lambda_2 \mathcal{L}_{\mathcal{D}} + \lambda_3 \mathcal{L}_{\mathcal{R}_\theta}
\end{equation*}
\textbf{D-TSN:}
For training and evaluating D-TSN, we perform 10 search iterations for each input state. We train it using the loss defined as:
\begin{equation*}
    \mathcal{L}_{\text{D-TSN}} = \mathbb{E}_{\tau} \Big[ \lambda_1 \mathcal{L}_{Q} + \lambda_2 \mathcal{L}_{\mathcal{D}} \Big] + \lambda_3 \mathcal{L}_{\mathcal{T}_\theta} + \lambda_4 \mathcal{L}_{\mathcal{R}_\theta}
\end{equation*}
Every method is trained with same datasets using their respective loss functions. We fine-tune the hyperparameters, $\lambda_1,\lambda_2,\lambda_3 \text{ and } \lambda_4$, using grid search on a log scale. A more comprehensive discussion of the baselines can be found in \cref{appendix_experiments_implementation_details}.

\begin{table}[t]
    \caption{Comparison of D-TSN with the baselines on Procgen using metrics \textit{Mean Scores} and \textit{Mean Z-score}.} \label{table_scores_on_procgen}
    \centering
    \resizebox{0.48\textwidth}{!}{
        \begin{tabular}{l|c|c|c|c}
            \toprule
            \multicolumn{1}{c|}{\multirow{2}{*}{Games}} & Model-free     & Model-based   & \multirow{2}{*}{TreeQN} & \multirow{2}{*}{\textbf{D-TSN}} \\
                                                        & Q-network      & Search        &                                                           \\
            \midrule
            bigfish                                     & \textbf{21.49} & 14.51         & 19.06                   & 20.80                           \\
            bossfight                                   & \textbf{8.77}  & 6.16          & 8.31                    & 8.53                            \\
            caveflyer                                   & 2.01           & 3.35          & \textbf{3.57}           & 3.53                            \\
            chaser                                      & 5.79           & 4.43          & 6.46                    & \textbf{6.66}                   \\
            climber                                     & 2.81           & \textbf{5.34} & 4.34                    & 5.20                            \\
            coinrun                                     & 5.02           & 6.50          & 5.07                    & \textbf{6.56}                   \\
            dodgeball                                   & 1.10           & 4.66          & 4.79                    & \textbf{5.52}                   \\
            fruitbot                                    & 14.00          & 10.50         & \textbf{15.59}          & 13.85                           \\
            heist                                       & 0.81           & 1.97          & 1.99                    & \textbf{2.12}                   \\
            jumper                                      & 3.39           & 4.06          & 3.42                    & \textbf{4.31}                   \\
            leaper                                      & 7.57           & 6.35          & 7.87                    & \textbf{8.05}                   \\
            maze                                        & 2.00           & \textbf{2.63} & 2.18                    & 2.51                            \\
            miner                                       & 1.42           & 1.55          & 1.55                    & \textbf{1.63}                   \\
            ninja                                       & 5.07           & 5.15          & 5.16                    & \textbf{5.88}                   \\
            plunder                                     & 12.77          & 9.67          & 12.50                   & \textbf{13.26}                  \\
            starpilot                                   & 15.53          & 13.57         & \textbf{16.98}          & 16.91                           \\
            \midrule
            \textit{Mean Z-score}                       & \textbf{-}     & 0.11          & 0.22                    & \textbf{0.31}                   \\
            \bottomrule
        \end{tabular}
    }
\end{table}

\subsection{Results}
\subsubsection{Navigation} 
We compare D-TSN against the baselines using evaluation metrics like success rate and collision rate, where success rate refers to the fraction of test levels completed by the agent, and collision rate refers to the fraction of levels failed due to collision with a wall. We report the evaluation results in the \cref{table_results_on_navigation}.

We observe that D-TSN outperforms the baselines on both navigation scenarios. Notably, when agents are trained on data from the 2-exits scenarios but are tested in the 1-exit scenario, D-TSN, with its powerful inductive bias, retains its performance. In stark contrast, the Model-free Q-network and TreeQN experience a substantial performance decline, which underscores their limited generalization ability. Model-based Search also registers a minor decrease in the success rate, reinforcing the importance of jointly optimizing the world model for enhanced robustness.

\subsubsection{Procgen}
We evaluate the performance of D-TSN and the baselines on all the games in the Procgen suite, comparing their performance in terms of mean Z-score and head-to-head wins. We also report the mean scores across 1000 episodes obtained by these methods in \cref{table_scores_on_procgen}. As Procgen games have different scales, we use model-free Q-network as baseline to compute a normalized score\footnote{We use Z-score because the baseline normalized score used in prior works~\cite{badia2020agent57,kaiser2019model,mittal2023expose} is problematic in our experiments, as described in \cref{appendix_experiments_bns_issue}.}, $\text{Z-score} = (\mu_\pi - \mu_B)/\sigma_B$, where $\mu_\pi \text{ and } \mu_B$ represent the mean scores obtained by the agent policy and the baseline policy respectively and $\sigma_B$ represents the standard deviation of the scores obtained by the baseline policy. 

\begin{table}
    \caption{Head-to-head comparison of D-TSN with the baselines on Procgen using metric \textit{number of games won}.}
    \label{table_results_on_procgen_1v1}
    \centering
    \begin{tabular}{l|c}
        \toprule
        \multicolumn{1}{c|}{\multirow{2}{*}{Baseline}} & Tree Search Network       \\
                                                       & (Games won / Total games) \\
        \midrule
        Model-free Q-network                           & 13 / 16                   \\
        Model-based Search                             & 14 / 16                   \\
        TreeQN                                         & 13 / 16                   \\
        \bottomrule
    \end{tabular}
\end{table}

It is important to note that the training data for Procgen was generated using a sub-optimal behavior policy, leading to noisy training signals in comparison to the relatively noise-free data in Navigation domain. Despite this, we observe that D-TSN reports a higher mean Z-score, averaged across the 16 Procgen games. This difference was particularly noticeable in games such as climber, coinrun, jumper, and ninja, which necessitate the planning of long-term action consequences, highlighting the role of the stronger inductive bias employed by D-TSN. In Procgen, Model-based Search lags behind both TreeQN and D-TSN, highlighting that for complex environments, joint optimization enables learning a robust world model. In head-to-head comparison listed in \cref{table_results_on_procgen_1v1}, D-TSN won in 13 games against the Model-free Q-network, 14 games against Model-based Search, and 13 games against TreeQN.

\begin{table}
    \caption{Comparison of D-TSN with its variants to evaluate the role of \textit{Auxiliary losses}, \textit{REINFORCE} term and \textit{Telescoping Sum}.}
    \label{table_ablations}
    \centering
    \resizebox{0.48\textwidth}{!}{
        \begin{tabular}{l|c|c}
            \toprule
            \multicolumn{1}{c|}{\multirow{3}{*}{Solver}} & \textit{Navigation}                       & \multirow{2}{*}{\textit{Procgen}} \\
                                                         & (1-exit)                                  &                  \\
            \cmidrule{2-3}
                                                         & Success Rate                              & Mean Z-score     \\
            \midrule
            \textbf{D-TSN}                               & \textbf{99.3\%} \scriptsize ($\pm$ 0.1\%) & \textbf{0.31}    \\
            \quad \textit{w/o Telescoping Sum}           & 98.5\% \scriptsize ($\pm$ 0.1\%)          & 0.28             \\
            \quad \textit{w/o REINFORCE term}            & 97.7\% \scriptsize ($\pm$ 0.2\%)          & 0.29             \\
            \quad \textit{w/o Auxiliary losses}          & 91.1\% \scriptsize ($\pm$ 0.3\%)          & -                \\
            \bottomrule
        \end{tabular}
    }
\end{table}

\subsection{Ablation Studies}
\subsubsection{The Impact of the REINFORCE Term and the Telescoping Sum Trick}
In this study, we assess the impact of both the REINFORCE term and the Telescoping Sum trick on D-TSN's performance. The results, presented in \cref{table_ablations}, show notable differences. Without the telescoping sum, D-TSN see a modest decrease in the success rate for Navigation (1-exit), moving from 99.3\% to 98.5\%. Similarly, the mean Z-score for Procgen dips to 0.28. The omission of the REINFORCE term also marks a decline, with the Navigation (1-exit) success rate landing at 97.7\% and the mean Z-score for Procgen dipping to 0.29. 

\subsubsection{The Contribution of Auxiliary Losses}
In this study, we explore the contribution of auxiliary losses to D-TSN's performance. As outlined in \cref{table_ablations}, the absence of these auxiliary losses lead to a more pronounced decline in the performance. Specifically, the success rate in the Navigation (1-exit) task drops significantly to 91.1\% compared to the 99.3\% with auxiliary losses.

\subsubsection{Enhancing Performance through Deeper Search}
In this study, we focus on the advantages of executing a deeper search in the D-TSN framework. To increase the search depth, we increased the number of search iterations during both the training and evaluation phases. Notably, we maintain an identical number of iterations in both phases to prevent any distribution shifts in the world models. The results, as listed in \cref{table_ablation_scaling_with_depth}, indicate that a higher number of iterations leads to an improvement in the success rate in both Navigation(2-exits) and Navigation(1-exit) scenarios. This improvement underscores the capability of D-TSN to exploit its search inductive bias effectively, suggesting that the system's performance can be scaled up further given additional computational resources.

\begin{table}[t]
    \caption{Comparison of D-TSN trained with different number of \textit{search iterations} (`n_itr') to evaluate the performance gain by performing deeper searches.}
    \label{table_ablation_scaling_with_depth}
    \centering
        \begin{tabular}{l|c|c}
            \toprule
            \multicolumn{1}{c|}{\multirow{2}{*}{Solver}} & \textit{Navigation } & \textit{Navigation } \\
                                                         & \textit{(2-exits)}   & \textit{(1-exit)}    \\
            \midrule
            D-TSN (n_itr=5)                             & 97.40\%              & 96.60\%              \\
            D-TSN (n_itr=10)                            & 99.00\%              & 99.30\%              \\
            D-TSN (n_itr=20)                            & 99.50\%              & 99.10\%              \\
            \bottomrule
        \end{tabular}
\end{table}
\begin{table}[t]
    \caption{Comparison of D-TSN and Model-based Search to highlight the robustness of the world model when performing deeper searches (`n_itr' refers to number of search iterations).}
    \label{table_ablation_robustness_of_world_model}
    \centering
    \resizebox{0.48\textwidth}{!}{
        \begin{tabular}{l|c|c}
            \toprule
            \multicolumn{1}{c|}{\multirow{2}{*}{Solver}} & \textit{Navigation } & \textit{Navigation } \\
                                                         & \textit{(2-exits)}   & \textit{(1-exit)}    \\
            \midrule
            D-TSN (\small{n_itr=10})                    & 99.00\%              & 99.30\%              \\
            D-TSN (\small{n_itr=20})                    & 99.30\%              & 99.40\%              \\
            D-TSN (\small{n_itr=50})                    & 99.70\%              & 99.60\%              \\
            \midrule
            Model-based Search (\small{n_itr=10})       & 93.20\%              & 86.90\%              \\
            Model-based Search (\small{n_itr=20})       & 91.10\%              & 84.60\%              \\
            Model-based Search (\small{n_itr=50})       & 89.50\%              & 80.40\%              \\
            \bottomrule
        \end{tabular}
    }
\end{table}

\subsubsection{Robustness of the World Model}
A key challenge in employing learned world models in online searches is addressing the compounding errors, which impacts the accuracy and effectiveness of online searches. However, this study shows that the joint optimization of both the world model and the search algorithm compensates for these inaccuracies, ensuring the world model is usable in deeper online searches. In this study, we train D-TSN with 10 iterations and subsequently evaluate it with an increased number of iterations. This approach allowed us to conduct deeper searches and verify the robustness of the world model. We compare it with Model-based Search, where the world model is trained independently of the search submodules. The findings, detailed in \cref{table_ablation_robustness_of_world_model}, demonstrate that the world model trained jointly with the search algorithm consistently outperforms the independently trained model, especially as the number of iterations increased.

\begin{table}[t]
    \caption{Comparison of average \textit{runtime} on Procgen for different methods to evaluate their computational requirements (`n_itr' refers to number of search iterations).}
    \label{table_ablation_runtime_on_procgen}
    \centering
    \resizebox{0.48\textwidth}{!}{
        \begin{tabular}{l|r|r}
            \toprule
            \multicolumn{1}{c|}{\multirow{2}{*}{Solver}} & \multicolumn{1}{c|}{Time Taken}    & \multicolumn{1}{c}{Total}         \\
                                        & \multicolumn{1}{c|}{Per Iteration} & \multicolumn{1}{c}{Training Time} \\
            \midrule
            Model-free Q-network        & 43ms                               & 4h 0m                             \\
            TreeQN (depth=2)            & 357ms                              & 23h 45m                           \\
            TreeQN (depth=3)            & Out-of-memory                      & Out-of-memory                     \\
            D-TSN (n_itr=5)                 & 153ms                              & 9h 30m                            \\
            D-TSN (n_itr=10)                & 294ms                              & 17h 30m                           \\
            \bottomrule
        \end{tabular}
    }
\end{table}

\subsubsection{Real-World Computation Speed Analysis}
To understand the computational demands of D-TSN and the baselines in real-world scenarios, we analyze the average time taken per training step (both forward and backward passes) and the total time taken for Model-free Q-network, TreeQN, and D-TSN in the Procgen environment as shown in \cref{table_ablation_runtime_on_procgen}. All models were trained on a 2080Ti GPU using PyTorch, with a batch size of 256. We observe that D-TSN provides the ability to conduct a deeper search while outpacing TreeQN in speed.
\section{Conclusion}
In this paper, we introduce \dtsn (D-TSN), a novel neural network architecture that conducts a fully differentiable online search in a latent state space. It features four learnable submodules that encode the input state into its latent state and conducts a best-first style search in the latent space using learned reward, transition, and value functions. Critically, the transition and reward modules are jointly optimized with the search algorithm, resulting in a robust world model that is useful for the online search and a search algorithm that can account for the errors in the world model. Furthermore, we note that a naive incorporation of best-first search could lead to discontinuity of the loss in the parameter space and address this issue by employing a stochastic tree expansion policy. We optimize the expected loss function using a REINFORCE-style objective function and propose a telescoping sum trick to reduce the variance of the gradient for this expected loss. We evaluate D-TSN against well-known baselines on Procgen games and a navigation task in an offline reinforcement learning setting to assess their sample efficiency and generalization capabilities. Our results show that D-TSN outperforms the baselines in both domains.

Nonetheless, the strength of the current implementation of D-TSN is currently limited to deterministic decision-making problems with a discrete action space. To cater to a broader spectrum of decision-making problems, there is a need to revamp the transition model to manage stochastic world scenarios, which we aim to address in our future works.

\section{Impact Statement}
This paper presents work whose goal is to advance the field of Machine Learning. Societal consequences depend on the specific application which uses the method developed here and will need to be considered on a case by case basis.

\bibliography{main}
\bibliographystyle{icml2024}

\newpage
\appendix
\onecolumn

\section{\dtsn Algorithm} \label{appendix_algorithm}
\subsection{\dtsn Pseudo-code}
\begin{algorithm}[H]
\linespread{1.2}\selectfont
\caption{Differentiable Tree Search (D-TSN)}
\label{algorithm_dtsn}
\SetKwComment{Comment}{//}{}
\KwIn{Input state, $s_{root}$}
\KwResult{Q-values, $Q(s_{root},a)$}
\BlankLine
$h_{root} \leftarrow \mathcal{E}_\theta (s_{root})$ \Comment*[r]{Encode $s_0$ to its latent state $h_0$}
$node_{root} \leftarrow initialise(h_{root})$ \Comment*[r]{Initialize root node}
$Open \leftarrow \{node_{root}\}$ \Comment*[r]{Initialize candidate set $Open$}
\tcp{\textbf{Expansion phase}}
$itr \leftarrow 0$\;
\Repeat{$itr$ < MAX\_ITR}{
    \ForEach{$node\in Open$}{
        $h_{node} \leftarrow getLatent(node)$ \Comment*[r]{Get latent state of $node$}
        $\Bar{V}(node) \leftarrow sumOfRewards(node) + \mathcal{V}_\theta(h_{node})$ \Comment*[r]{Compute path values for open nodes}
    }
    $\pi_{tree} \leftarrow \text{softmax}_n \Big(\Bar{V}(n) \Big)$ \Comment*[r]{Compute the tree expansion policy}
    $node^* \leftarrow \text{sample} (\pi_{tree})$ \Comment*[r]{Sample the node to expand}
    \ForEach{$a \in Actions$;}{
        $h_{child} \leftarrow T_{\theta}(h_{node^*},a)$\; \Comment*[r]{Compute the next latent state}
        $r_{child} \leftarrow \mathcal{R}_\theta (h_{node^*},a)$\; \Comment*[r]{Compute reward for the transition}
        $createNode(h_{child}, r_{child})$\; \Comment*[r]{Create the corresponding child node}
    }
    $Open \leftarrow Open \cup \{child_a | child_a = getChild(node^*,a); \forall a \in A\} - node^*$\; \Comment*[r]{Update the open set}
    $itr \leftarrow itr + 1$\;
}

\tcp{\textbf{Backup phase}}
\ForEach{$node \in Tree$, iterating from leaf nodes to the root node;}{
    \If{$node$ is a leaf;\ }{
        $h_{node} \leftarrow getLatent(node)$ \Comment*[r]{Get latent state of $node$}
        $V(node) = \mathcal{V}_\theta(h_{node})$ \Comment*[r]{Compute value using Value module}
    }
    \Else{
        \ForEach{$a \in Actions$;}{
            $node_{child[a]} \leftarrow getChild(node,a)$ \Comment*[r]{Get child of $node$ that corresponds to action $a$}
            $h_{node} \leftarrow getLatent(node)$ \Comment*[r]{Get latent state of $node$}
            $r_{node} \leftarrow \mathcal{R}_\theta(h_{node},a)$ \Comment*[r]{Get reward using Reward module}
            $Q(node,\ a) \leftarrow  r_{node} + V\left( node_{child[a]}\right)$ 
        }
        $V(node) \leftarrow \max_a Q(node,a)$
    }
}
\Return{$Q(node_{root})$} \Comment*[r]{Return Q-value of the root node}
\end{algorithm}
\vspace{10pt}

\newpage
\subsection{Pseudo-code Explanation}
D-TSN initializes the search by converting the input state to its latent representation and proceeds in two phases: Expansion and Backup as detailed in the following sections.

\subsubsection{Initialization}
Given an input state \(s_{\text{root}}\), the algorithm begins by encoding this state into its latent representation \(h_{\text{root}}\) using an encoder \(\mathcal{E}_\theta\). This latent representation serves as the root node of the search tree.

\subsubsection{Expansion Phase}
\begin{itemize}
    \item The algorithm initiates a set of candidate nodes, termed `Open', starting with the root node. 
    \item In each iteration, the algorithm considers every node in the `Open' set, retrieves its latent state, and computes an interim value \(\Bar{V}(\text{node})\), which combines the cumulative rewards of the node's path with a value estimation from the value module \(\mathcal{V}_\theta\).
    \item Using the path values of nodes in `Open', the tree expansion policy, \(\pi_{\text{tree}}\), is computed. From this policy, a node, \(node^*\), is sampled for expansion.
    \item Each possible action from \(node^*\) results in the creation of a child node. This is achieved by leveraging the transition module \(\mathcal{T}_\theta\) to predict the latent state of the child and the reward module \(\mathcal{R}_\theta\) to determine the associated reward. After expansion, \(node^*\) is removed from `Open' and its children are added.
    \item This expansion process continues until a predetermined number of iterations, \texttt{MAX\_ITR}, is reached.
\end{itemize}

\subsubsection{Backup Phase}
\begin{itemize}
    \item Starting from the leaf nodes, the algorithm backpropagates value estimates to the root.
    \item For each leaf node, its value is directly computed from the value module \(\mathcal{V}_\theta\). For non-leaf nodes, the Q-value for each action is estimated by combining the reward for that action with the value of the corresponding child node.
    \item The value of a non-leaf node is set to the maximum Q-value among its actions.
\end{itemize}
    
\subsubsection{Output} 
The algorithm finally returns the Q-values associated with the root node, \(Q(node_{\text{root}})\), providing an estimation of the value of taking each action from the initial state.

This explanation provides a high-level view of the D-TSN algorithm's operation. By breaking down the search process into expansion and backup phases, the pseudo-code highlights how D-TSN incrementally builds the search tree and then consolidates value estimates back to the root.
\newpage
\section{Continuity of the Loss Function in TreeQN} \label{appendix-section-continuity}
In this section, we show that the Q-function represented by TreeQN is continuous in network's parameter space. 

Suppose we have two functions $f(x)$ and $g(x)$ which are continuous at any point $c$ in their domains.

\begin{lemma} \label{lemma-continuity-composition}
\textbf{Continuity of Composition}: The composition of two continuous functions, denoted as $f(g(x))$, retains continuity. (Theorem 4.7 in \cite{rudin1976principles})
\end{lemma}

\begin{lemma} \label{lemma-continuity-sum}
\textbf{Continuity of Sum operation}: The result of adding two continuous functions, expressed as $f(x) + g(x)$, is a continuous function. (Theorem 4.9 in \cite{rudin1976principles})
\end{lemma}

\begin{lemma} \label{lemma-continuity-max}
\textbf{Continuity of Max operation}: Applying Max over two continuous functions, expressed as $\max(f(x), g(x))$, results in a function that is continuous.
\end{lemma}
\begin{proof} \label{proof-continuity-max}
Consider a function $h(x) = \max (f(x), g(x))$, and we aim to demonstrate that $h(x)$ is continuous. Now, $h(x)$ can be expressed as a combination of continuous functions: 
\begin{equation*}
    h(x) = \dfrac{f(x) + g(x) + |f(x) - g(x)|}{2}
\end{equation*}

Since sums and absolute values of continuous functions are continuous~\citep{rudin1976principles}, $h(x)$ is continuous. Hence, the maximum of two continuous functions is also a continuous function.
\end{proof}

Rewriting the \cref{theorem_continuity_of_q_function} with proof below:
\begin{theorem}
Given a set of parameterised modules that are continuous in the parameter space $\theta$, the Q-function computed by fully expanding a search tree to a fixed depth `$d$' by composing these modules and backpropagating the children values using addition and max operations is continuous in the parameter space $\theta$.
\end{theorem}

\begin{proof}
Let us represent the set of parameters, encoder module, transition module, reward module, and value module respectively as $\theta$, $\mathcal{E}_\theta$, $\mathcal{T}_\theta$, $\mathcal{R}_\theta$, and $\mathcal{V}_\theta$. These submodules are assumed to be composed of simple neural network architectures, comprising linear and convolutional layers, with the Rectified Linear Unit (ReLU) serving as the activation function. These submodules, therefore, are continuous within the parameter space.

We can subsequently rewrite the Q-value as the output of a full tree expansion as follows:
\begin{dmath}
    Q(s_0,a_0)  
    = Q(h_0,a_0)
    = r_0 + V(h_1)
\end{dmath}
where,
\begin{align} 
    h_0 &= \mathcal{E}_\theta(s_0) \label{proof-equation-encoder} \\
    r_t &= \mathcal{R}_\theta(h_t,a_t) \label{proof-equation-reward} \\
    h_{t+1} &= \mathcal{T}_\theta(h_t,a_t) \label{proof-equation-transition}\\
    Q(h_t,a_t) &= \mathcal{R}_\theta(h_t,a_t) + V(h_{t+1}) \\
    V(h_t) &= \label{proof-equation-value} 
    \begin{cases}
        \mathcal{V}_\theta(h_t) & \text{if } h_t \text{ is a leaf} \\
        \max_a \left(Q(h_t,a)\right) & \text{otherwise }
    \end{cases}
\end{align}
Given that $\mathcal{E}_\theta$, $\mathcal{R}_\theta$, and $\mathcal{T}_\theta$ are continuous, $h_0$, $r_t$, and $h_{t+1}$ in \cref{proof-equation-encoder,,proof-equation-transition} are similarly continuous (derived from \cref{lemma-continuity-composition}).

Further, $V(h_t)$ in \cref{proof-equation-value} can either be $\mathcal{V}_\theta(h_t)$, if $h_t$ is a leaf node, or $\max_a(Q(h_t,a))$ otherwise. In the first scenario, $V(h_t)$ retains continuity by assumption. In the second scenario, if $Q(h_t,a_t)$ is continuous, then $V(h_t)$ remains continuous as per \cref{lemma-continuity-max}.

Now, we show that $Q(h_t,a_t)$ is continuous using a recursive argument that for any node in the search tree, if the Q-values of all its child nodes are continuous, then its Q-value is also continuous. Q-value of an internal tree node $h_t$ can be written as $Q(h_t,a_t) = \mathcal{R}_\theta(h_t,a_t) + V(h_{t+1})$, where $h_{t+1}=\mathcal{T}_\theta(h_t,a_t)$ is the child node of $h_t$. Considering the base case, when the $h_{t+1}$ is a leaf node, then $Q(h_t,a_t) = \mathcal{R}_\theta(h_t,a_t) + \mathcal{V}_\theta(h_{t+1})$, which is continuous as per ~\cref{lemma-continuity-sum}. Consequently, $V(h_t)$ maintains continuity. Applying this logic recursively, the Q-value $Q(h_t,a_t)$ of all the tree nodes maintains continuity.

Thus, we can decompose the Q-value at the root node, $Q(s_0,a_0)$, as a composition of continuous functions, ensuring that the output Q-value, $Q(s_0,a_0)$, is continuous.
\end{proof}

Consequently, the Q-function represented by TreeQN is continuous in the parameter space. Furthermore, when this Q-function is combined with a continuous loss function, the resulting loss is continuous in the network's parameter space. 
\section{Derivation of the Gradient of the Expected Loss Function} \label{appendix_section_derivation_of_gradient}

Let us represent the Q-values predicted by D-TSN as $Q_\theta(s,a | \tau)$, which depends on the final tree $\tau$ sampled after $T$ trials of the online search. Let us denote the corresponding loss function on this output Q-value as $\mathcal{L}\Big(Q_\theta(s,a|\tau)\Big)$. Our objective is to compute the gradient of the expected loss value, averaging over trees sampled.

The gradient of expected loss, considering the expectation over the sampled trees, is derived as:
\begin{dmath*}
    \nabla_\theta L(s,a) 
    = \nabla_\theta \mathbb{E}_{\tau} \bigg[\mathcal{L}\Big(Q_\theta(s,a|\tau)\Big)\bigg] \\~\\
    = \nabla_\theta \sum_\tau \pi_\theta(\tau) \mathcal{L}\Big(Q_\theta(s,a|\tau)\Big) \\~\\
    = \sum_\tau \nabla_\theta \bigg[\pi_\theta(\tau) \mathcal{L}\Big(Q_\theta(s,a|\tau)\Big) \bigg] \\~\\
    = \sum_\tau \mathcal{L}\Big(Q_\theta(s,a|\tau)\Big) \nabla_\theta \pi_\theta(\tau) + \sum_\tau  \pi_\theta(\tau) \nabla_\theta \mathcal{L}\Big(Q_\theta(s,a|\tau)\Big) \\~\\
    = \sum_\tau  \pi_\theta(\tau) \mathcal{L}\Big(Q_\theta(s,a|\tau)\Big) \nabla_\theta \log \pi_\theta(\tau) + \sum_\tau  \pi_\theta(\tau) \nabla_\theta \mathcal{L}\Big(Q_\theta(s,a|\tau)\Big) \\~\\
    = \mathbb{E}_{\tau} \bigg[ \mathcal{L}\Big(Q_\theta(s,a|\tau)\Big) \nabla_\theta \log \pi_\theta(\tau) + \nabla_\theta \mathcal{L}\Big(Q_\theta(s,a|\tau)\Big) \bigg] \\~\\
    = \mathbb{E}_{\tau} \bigg[ \mathcal{L}\Big(Q_\theta(s,a|\tau)\Big) \nabla_\theta \log \prod_{t=1}^T \pi_\theta(n_t|\tau_t) + \nabla_\theta \mathcal{L}\Big(Q_\theta(s,a|\tau)\Big) \bigg] \\~\\
    = \mathbb{E}_{\tau} \bigg[ \mathcal{L}\Big(Q_\theta(s,a|\tau)\Big) \sum_{t=1}^T \nabla_\theta \log \pi_\theta(n_t|\tau_t) + \nabla_\theta \mathcal{L}\Big(Q_\theta(s,a|\tau)\Big) \bigg]
\end{dmath*}

Leveraging the telescoping sum trick, as elaborated in \cref{section_method_telescoping_sum}, the gradient of the expected loss can be rewritten as a lower-variance estimate:

\begin{equation*}
    \nabla_\theta L(s,a) = \mathbb{E}_{\tau} \bigg[  \sum_{t=1}^T \nabla_\theta \log \pi_\theta(n_t|\tau_t) R_t + \nabla_\theta \mathcal{L}\Big(Q_\theta(s,a|\tau)\Big) \bigg]
\end{equation*}

where
\begin{align*}
    R_t &= \sum_{i=t}^T r_i = \mathcal{L}_T - \mathcal{L}_{t-1} \\
    \mathcal{L}_t &= \mathcal{L}\Big(Q_\theta(s,a|\tau_{t})\Big), \text{ representing the loss value after the } t^{th} \text{ search iteration.}
\end{align*}

In practice, we utilize the single-sample estimate for the expected gradient, as elaborated in~\citet{SchulmanHWA15}
\begin{figure}[ht]
    \begin{center}
        \includegraphics[width=\linewidth]{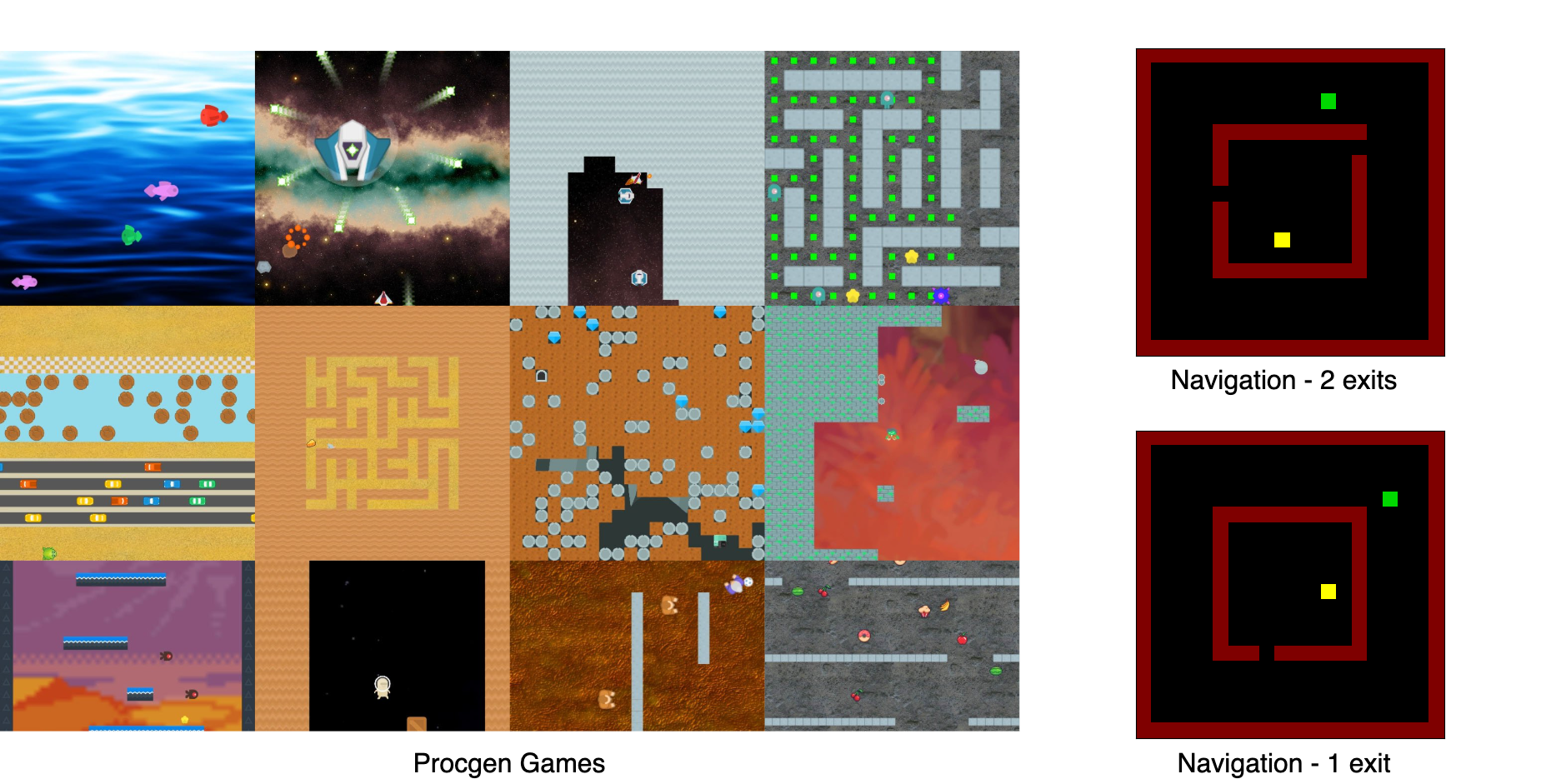}
    \end{center}
    \caption{A sample visualization of Procgen games (left) and Grid Navigation (right).}
    \label{figure_domains}
\end{figure}

\begin{figure}[ht]
    \begin{center}
        \includegraphics[width=\linewidth]{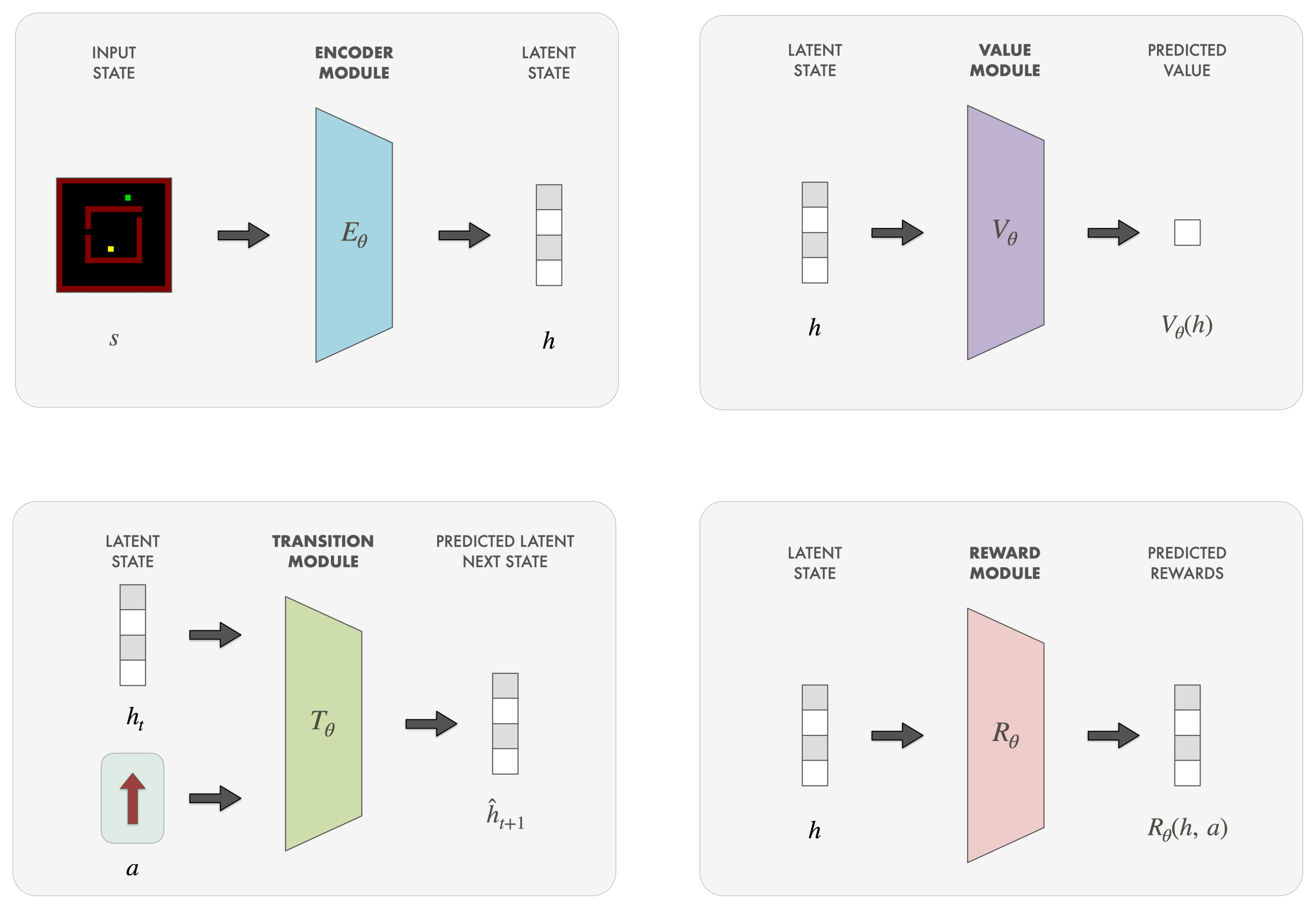}
    \end{center}
    \caption{An illustration of the learnable submodules in \dtsn} \label{figure_illustrations_submodules}
\end{figure}

\begin{figure}[ht]
\centering
\begin{minipage}{.5\textwidth}
  \centering
  \includegraphics[width=0.93\linewidth]{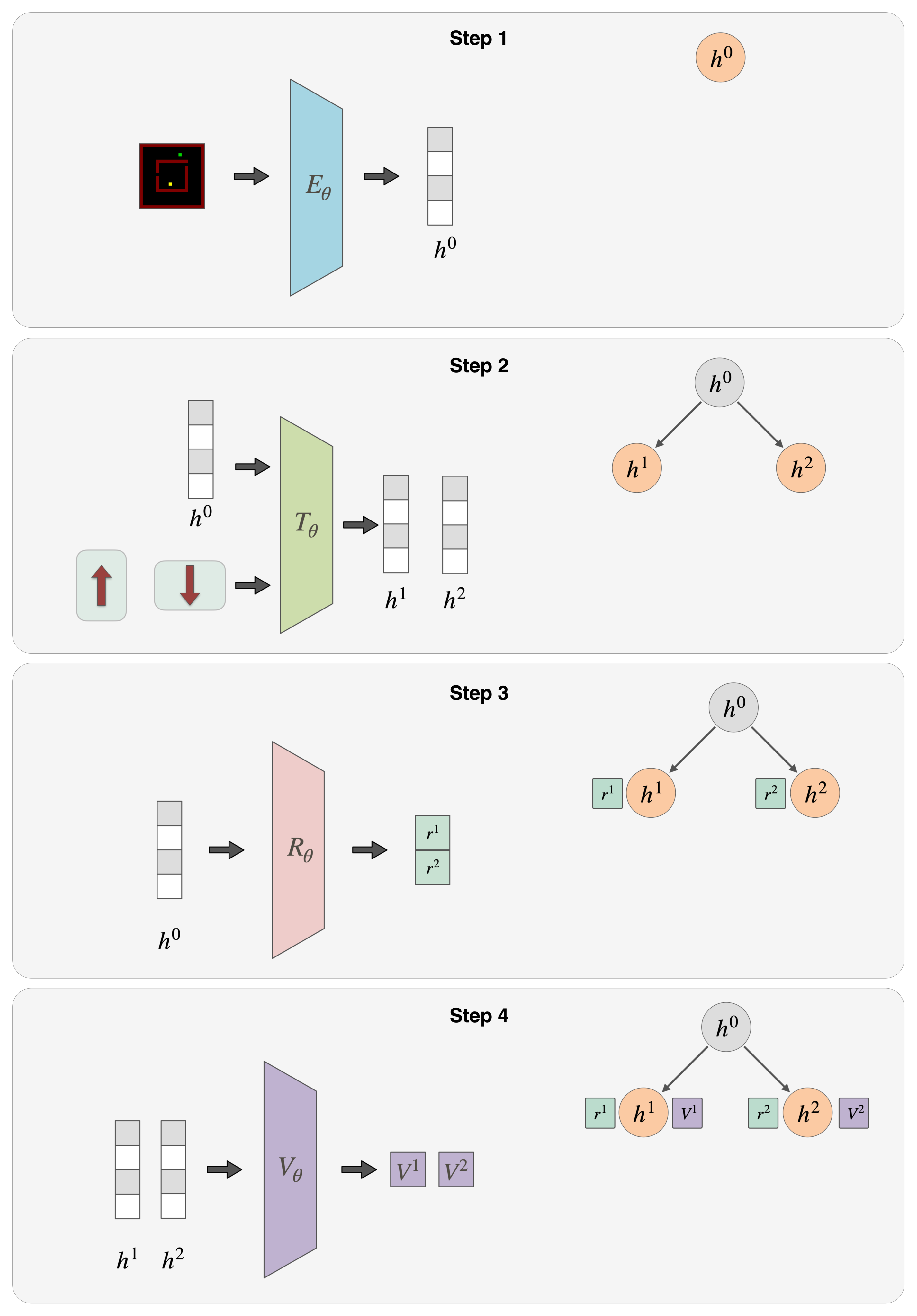}
\end{minipage}%
\begin{minipage}{.5\textwidth}
  \centering
  \includegraphics[width=\linewidth]{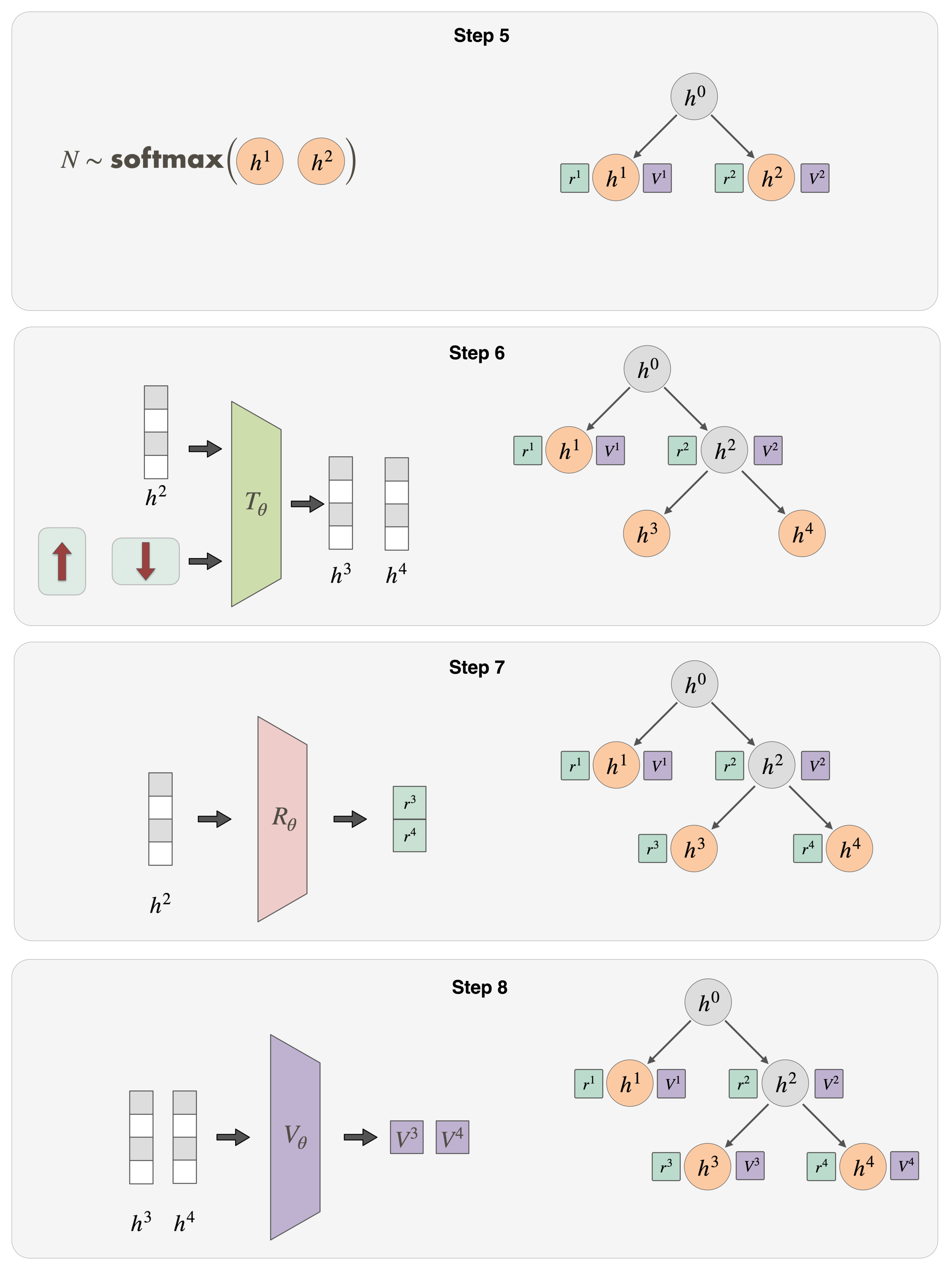}
\end{minipage}
\caption{An illustration of the Expansion Phase in \dtsn} \label{figure_illustrations_expansion}
\end{figure}

\begin{figure}[ht]
    \begin{center}
        \includegraphics[width=0.5\linewidth]{figures/4_backup.png}
    \end{center}
    \caption{An illustration of the Backup Phase in \dtsn} \label{figure_illustrations_backup}
\end{figure}

\begin{figure}[ht]
    \begin{center}
        \includegraphics[width=0.8\linewidth]{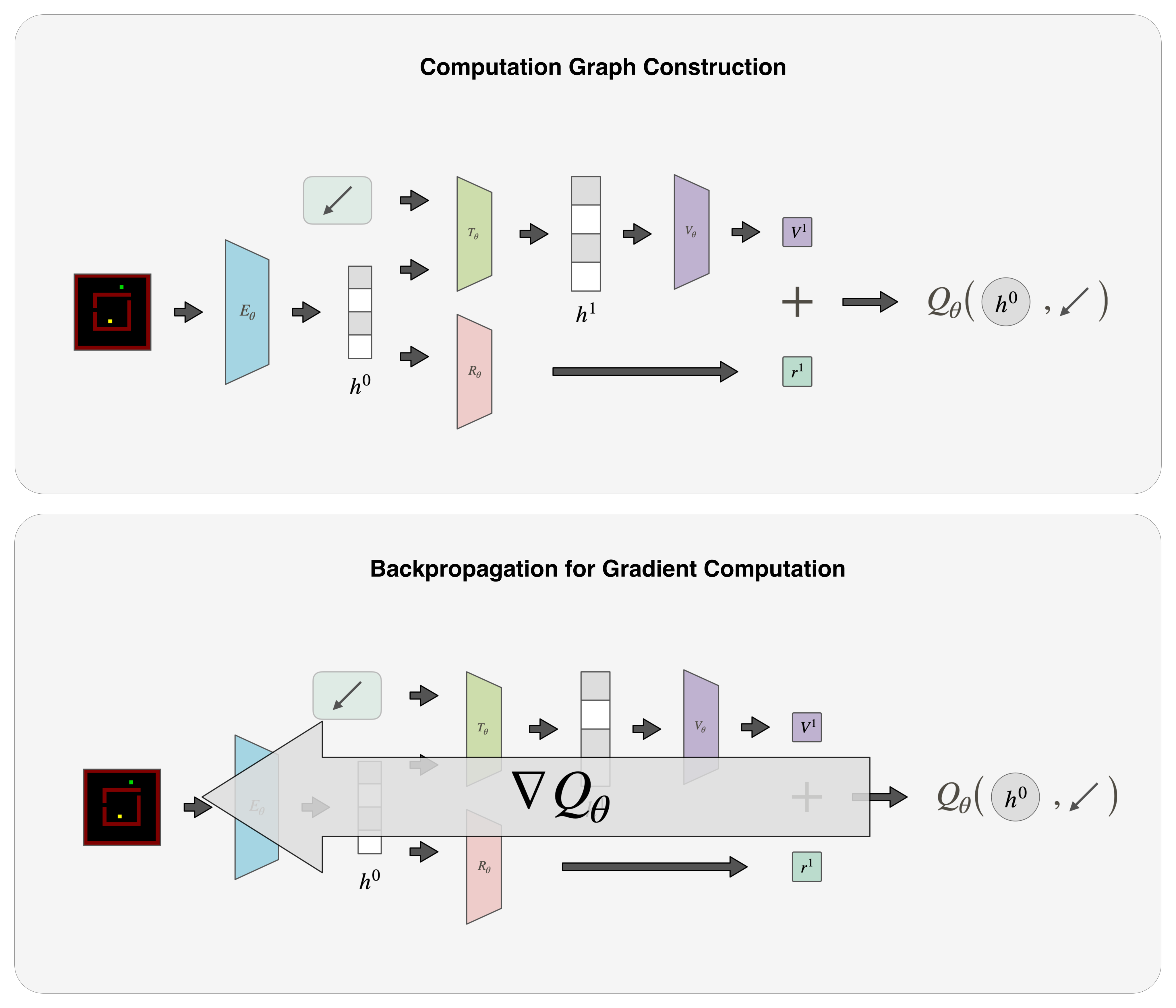}
    \end{center}
    \caption{An illustration of the computation graph construction in \dtsn} \label{figure_illustrations_graph}
\end{figure}
\section{Experiments} \label{appendix_experiments}
\subsection{Test Domains} \label{appendix_experiments_domains}
In our evaluations, we choose two distinct domains to assess and compare the sample complexity and generalization capabilities of various methods: a grid navigation task and the Procgen games. These domains are selected due to their diverse challenges, offering a comprehensive evaluation spectrum. A visual representation of these domains can be viewed in \cref{figure_domains}.

\subsubsection{Navigation} 
The grid navigation task serves as a foundational test that mimics the challenges a robot might face when navigating in a 2D grid environment. This environment provides both a quantitative metric and a qualitative visualization to understand an agent's capacity to generalize its policy. Specifically, this task involves a $20\times20$ grid with a central hall. At the beginning of each episode, the robot is positioned at a random point within this central hall. Simultaneously, a goal position is sampled randomly at a location outside the hall, challenging the robot to find its way out and reach this target. There are two variations of this task. The first provides the robot with a single exit from the central hall, while the second offers two exits. The single-exit hall scenario is similar to the two-exit scenario but requires a longer-horizon planning to successfully evade the walls to exit the hall and reach the goal.

\subsubsection{Procgen} 
Procgen is a unique suite consisting of 16 game-like environments, each of which is procedurally generated. This means that they are designed to present slightly altered levels every time they are played. Such design intricacy makes Procgen an ideal choice to test an agent's generalization capabilities. It stands in contrast to other commonly used testing suites, like the renowned Atari 2600 games~\citep{mnih2013playing,mnih2016asynchronous,badia2020agent57}. The diverse array of environments within Procgen emphasizes the pivotal role of robust policy learning. The environmental diversity in Procgen underlines the importance of robust policy learning for successful generalization. The open-source code for the environments is publicly accessible at \url{https://github.com/openai/procgen}.

\subsection{Learning Setup} \label{appendix_experiments_learning_setup}
D-TSN can serve as a drop-in replacement for popular Convolutional Neural Network (CNN) architectures. It can be trained using both online and offline reinforcement learning algorithms. In this work, we employ the offline reinforcement learning (Offline-RL) framework to focus on the sample complexity and generalization capabilities of D-TSN when compared with the baselines. 

Offline-RL, often referred to as batch-RL, is the scenario wherein an agent learns its policy solely from a fixed offline dataset of experiences, without further interactions with the environment. Offline-RL framework poses significant challenges, especially when it comes to the generalization capabilities of methods, even for seemingly straightforward tasks like navigation. Consider, for instance, that the optimal policy is employed to collect experiences from an environment. It would predominantly select the optimal actions at every state. This means that only a fraction of the vast state-action space would be covered in the dataset. As a result, the model might not learn about many interactions, such as what happens when it collides with a wall, due to the limited training data on such environment interactions. When this world model is employed in an online search scenario, the search process might unknowingly venture into out-of-distribution state space. These explorations, stemming from the model's limited generalization capabilities, can lead to overly optimistic value predictions, subsequently affecting the Q-values computation at the root node. In practical terms, the online search might mistakenly believe it can travel through a wall to reach its goal more quickly and would then run into it.

However, the D-TSN design offers a solution to this problem by jointly optimizing both the world model and the online search. During the training phase, when the online search strays into the out-of-distribution states, it might overestimate the value of these states. This overestimation would then influence the final output after the search. When such a mismatch between the post-search output and the expert action is detected, the gradient descent algorithm adjusts these overestimated values, effectively lowering them to align the final Q-values more closely with the expert action. As a direct consequence of this, by the end of the training phase, the search process will have effectively learned to ignore these the out-of-distribution states. Even if such expansion does occur due to $\epsilon$-greedy exploration during inference, the predicted value for the out-of-distribution state will be small and thus will have negligible impact on the Q-values of the root node due to the max operation during backup.

\subsection{Training Datasets} \label{appendix_experiments_training_datasets}
We employ a separate behavior policy to collect the offline training dataset. Here, the behavior policy can be \textit{optimal or sub-optimal}. An optimal policy generates a dataset with lower noise and a cleaner training signal, leading to a stable learning process. In contrast, a sub-optimal policy produces a noisier dataset, which consequently restricts the quality of the policy that the agent can learn. The selection of the behavior policy depends on domain-specific requirements and the resources available for data collection. This work explores datasets generated using both optimal (in the Navigation domain) and sub-optimal (in the Procgen domain) behavior policies.

The training dataset, $\mathcal{D}$, consists of trajectories generated using the behavior policy $\pi_B$, where each trajectory, $\tau_i$, is defined as a series of $T$ tuples, each comprising the state observed, action taken, reward observed, and the corresponding Q-value of the observed state, denoted as:
\begin{equation*}
    \tau_i = \Big\{(s_{t,i},\ a_{t,i},\ r_{t,i},\ Q_{t,i})\Big\}_{t=0}^{T}
\end{equation*}
The Q-value for state $s_{t,i}$ can be computed by adding the rewards obtained in the trajectory from timestep $t$ onwards, i.e. 
\begin{align*}
    Q_{t,i} &= Q^{\pi_B} (s_{t,i})\\
            &= r_{t,i} + r_{t+1,i} + r_{t+2,i} + ... r_{T,i} \\
            &= \sum_{k=t}^T {r_{k,i}}
\end{align*}
In order to evaluate the sample complexity and generalization capabilities of each method, we collected a small number of $1000$ trajectories for each test domain for our experiments.

\subsubsection{Navigation}
For our navigation task, which is relatively small in size, we are able to compute the optimal policy for any given state and configuration. We employ the value iteration algorithm for this purpose, as detailed by~\citet{sutton2018reinforcement}. At the beginning of each episode, we formulate a random passage through the central hall. Subsequently, we also randomly determine the starting position of the robot within the hall and its goal position outside of it. We collect a total of $1000$ expert trajectories for training. Each of these trajectories incorporates a sequence comprising states, actions, rewards, and Q-values observed throughout the episode.

\subsubsection{Procgen}
When it comes to Procgen, even though there isn't a public repository of pre-trained models, there exists an open-source code base for Phasic Policy Gradient (PPG) \citep{cobbe2020ppg}. With this in hand, we could effectively train a decent, \textit{but sub-optimal}, policy for every individual Procgen game starting from the ground up. We rely on the default set of hyperparameters for training a specific policy for each game. Using these policies, we gather sample trajectories for $1000$ successful episodes. Just like in the case of Navigation, each of these trajectories represents a sequence comprising states, actions, rewards, and Q-values observed throughout the episode.

\subsection{Loss Functions} \label{appendix_experiments_loss_functions}
As mentioned in the previous section, D-TSN can be trained using various online and offline RL methods. In this work, we are focusing on offline-RL framework and use Behavior Cloning to train the parameters of D-TSN. Behavior Cloning is a type of supervised learning where the objective is to make the agent's prediction closely approximate the actions taken by the behavior policy in the states collected in the training dataset. To achieve this, we minimize the mean squared error between the predicted and target Q-values. This loss, denoted as $\mathcal{L}_{Q}$, is expressed as:
\begin{equation} \label{appendix_equation_mse_loss}
\mathcal{L}_{Q} = \mathbb{E}_{(s_{t,i},\ a_{t,i},\ Q_{t,i}) \sim \mathcal{D}} \Big(
Q_\theta(s_{t,i},\ a_{t,i}) - Q_{t,i} \Big)^2
\end{equation}

Moreover, during the online search, the transition, reward, and value networks operate on the latent states. Consequently, it is important to ensure that the input to these networks is of a consistent scale, as suggested in \citep{schwarzer2020data,ye2021mastering}. To achieve this, we apply hyperbolic tangent (Tanh) normalization on the latent states, thereby adjusting their scale to fall within the range $(-1,1)$.
\begin{align*}
    h   &= tanh(x) \\
        &= \dfrac{e^{x} - e^{-x}}{e^{x} + e^{-x}} \quad \in (-1,1)
\end{align*}

\subsubsection{Auxiliary Loss for Out-of-Distribution Actions}
In the offline-RL setting, there's a risk that Q-values for out-of-distribution actions might be overestimated \citep{kumar2020conservative} as the behavior policy can only cover a limited part of the state-action distribution. To address this, we incorporate an additional CQL \citep{kumar2020conservative} loss which encourages the agent to adhere to actions observed within the training data distribution. This loss, $\mathcal{L}_{\mathcal{D}}$, is defined as:
\begin{equation}
\label{appendix_equation_cql_loss}
\mathcal{L}_{\mathcal{D}} = \mathbb{E}_{(s_{t,i},\ a_{t,i}) \sim \mathcal{D}} \Big( \log \sum_{a'} \exp \Big( Q_\theta(s_{t,i},\ a') \Big) - Q_\theta(s_{t,i},\ a_{t,i}) \Big)
\end{equation}

\subsubsection{Auxiliary Loss for Consistency in the World Model}
In order to avoid overburdening the latent states with extraneous information required to reconstruct the original input states, like in Model-based RL methods, we utilize self-supervised consistency loss functions as described in~\citep{schwarzer2020data,ye2021mastering}. These functions aid in maintaining consistency within the transition and reward networks. For example, let us assume a state $s_{t}$ and the subsequent state $s_{t+1}$ resulting from action $a_{t}$. The latent state representations for the environment states $s_{t}$ and $s_{t+1}$ can be computed as $h_{t}$ and $h_{t+1}$ respectively. The latent state encoding $\hat{h}_{t+1}$ can be predicted using the transition module, $\hat{h}_{t+1} = \mathcal{T}_\theta(h_{t},\ a_{t})$. We minimize the squared error between the latent representation $h_{t+1}$ and $\hat{h}_{t+1}$, to ensure that the transition function $\mathcal{T}_\theta$ provides consistent predictions for the transitions in the latent space. In accordance with the approach detailed in \citep{ye2021mastering}, we use a separate encoding network, referred to as the target encoder, to compute target representations.

\begin{equation} \label{appendix_equation_transition_loss}
\mathcal{L}_{\mathcal{T}_\theta} = \mathbb{E}_{(s_{t,i},\ a_{t,i},\ s_{t+1,i}) \sim \mathcal{D}} \left[ \left(\hat{h}_{t+1,i} - h_{t+1,i} \right)^2 \right]
\end{equation}

where 

\begin{align*}
    \hat{h}_{t+1,i} &= \mathcal{T}_\theta(h_{t,i},\ a_{t,i}) \\
    h_{t,i} &= \mathcal{E}_\theta(s_{t,i}) \\
    h_{t+1,i} &= \mathcal{E}_{\theta'} (s_{t+1,i})
\end{align*}

The parameters of the target encoder, $\theta'$, are updated using an exponential moving average of the parameters of the base encoder, $\theta$, as follows. 

\begin{equation*}
    \theta' \leftarrow \alpha \ \theta' + (1-\alpha) \ \theta
\end{equation*}

Notably, we refrain from adding projection or prediction networks, as done in \cite{schwarzer2020data} and \cite{ye2021mastering}, prior to calculating the squared difference as we did not observe any improvement by doing so.

Further, we also seek to minimize the mean squared error between the predicted reward $\mathcal{R}_\theta(h, a)$ and the actual reward observed $r$ in the training dataset $\mathcal{D}$.

\begin{equation} \label{appendix_equation_reward_loss}
\mathcal{L}_{\mathcal{R}_\theta} = \mathbb{E}_{(s_{t,i},\ a_{t,i},\ r_{t,i}) \sim \mathcal{D}} \bigg[ \Big( \mathcal{R}_\theta(h_{t,i},\ a_{t,i}) - r_{t,i} \Big)^2 \bigg]
\end{equation}

\subsubsection{Final Loss Function}
The final loss function to train D-TSN is a combination of Behavior Cloning loss and all the auxiliary losses defined above. It is given by:
\begin{equation}
\label{equation_final_loss}
\mathcal{L} = \lambda_1 \mathcal{L}_{Q} + \lambda_2 \mathcal{L}_{\mathcal{D}} + \lambda_3 \mathcal{L}_{\mathcal{T}_\theta} + \lambda_4 \mathcal{L}_{\mathcal{R}_\theta}
\end{equation}

where $\lambda_1, \lambda_2, \lambda_3$ and $\lambda_4$ serve as weighting hyperparameters.

Although, in our discussion above, we define the loss functions in the context of training D-TSN, the baseline methods also utilize a combination of the loss functions from \cref{appendix_equation_mse_loss,,appendix_equation_cql_loss,,appendix_equation_transition_loss,,appendix_equation_reward_loss} depending upon their specifications.

\subsection{Baselines} \label{appendix_experiments_baselines}
To evaluate the efficacy of Tree Search Network, we benchmark it against the following prominent baselines:
\begin{itemize}
    \item \textbf{Model-free Q-network}: This allows us to assess the significance of integrating the inductive biases into the neural network architecture.    
    \item \textbf{Model-based Search}: In this baseline, we assess the benefits derived from the joint optimization of the world model and the search algorithm by training the world models and value module independently of each other and utilizing them for online search.
    \item \textbf{TreeQN}: This comparison helps in highlighting the advantages of using a more advanced search algorithm that can execute a deeper search while maintaining similar computational constraints. 
\end{itemize}

\subsection{Implementation Details} \label{appendix_experiments_implementation_details}
In an effort to assess the distinctive elements of each method's design, we ensure uniformity in the number of parameters across all agents. This is achieved by integrating the submodules from D-TSN into the network architecture of each baseline. However, while the number of parameters are consistent, the way in which these submodules are utilized to construct the computation graph varies among the baselines. We provide their implementation details below:

\subsubsection{\dtsn}
D-TSN utilizes its submodules in alignment with the best-first search algorithm presented in \cref{section_method_tree_search}. For our empirical evaluations, we set the maximum limit for search iterations at 10. Throughout the training process, the computation graph, formulated via online search, is optimized to accurately predict the Q-values. This optimization serves a dual purpose: it not only refines the Q-value predictions but also facilitates robust learning for the submodules when they are employed in context of online search. The loss function utilized for training is:
\begin{equation*}
    \mathcal{L}_{D-TSN} = \mathbb{E}_{\tau} \Big[ \lambda_1 \mathcal{L}_{Q} + \lambda_2 \mathcal{L}_{\mathcal{D}} \Big] + \lambda_3 \mathcal{L}_{\mathcal{T}_\theta} + \lambda_4 \mathcal{L}_{\mathcal{R}_\theta}
\end{equation*}
We compute the gradient for this loss as described in \cref{equation_tree_policy_gradient_telescopic}.

\subsubsection{Model-free Q-network}
In this baseline, the submodules are utilized to perform a one-step look-ahead search. The input state at the root node undergoes an expansion using the world model, and Q-values are computed using the Bellman equation represented as $Q(s,a) = Rew(h,a) + Val(h')$, where $h = Enc(s)$ and $h' = Tr(h,a)$. Intriguingly, this structure does encapsulate a basic inductive bias via the 1-step look-ahead search. However, in keeping with its model-free characteristic, auxiliary losses aren't employed for training the transition and reward model. The loss function for this model is:
\begin{equation*}
    \mathcal{L}_{QNet} = \lambda_1 \mathcal{L}_{Q} + \lambda_2 \mathcal{L}_{\mathcal{D}}
\end{equation*}

\subsubsection{Model-based Search}
For this approach, we employ the best-first algorithm showcased in \cref{algorithm_dtsn}. However, there's a difference: the world model and the value module are trained independently, each focusing on their specific objectives. As outlined in \cite{ye2021mastering}, we incorporate self-supervised consistency losses defined in \cref{appendix_equation_transition_loss,,appendix_equation_reward_loss} as they improve the online search, even in cases where the world model is not jointly trained with the online search. The Q-values used for training are computed directly using the value module without performing online search during training. The loss function used for this approach is:
\begin{equation*}
    \mathcal{L}_{Search} = \lambda_1 \mathcal{L}_{Q} + \lambda_2 \mathcal{L}_{\mathcal{D}} + \lambda_3 \mathcal{L}_{\mathcal{T}_\theta} + \lambda_4 \mathcal{L}_{\mathcal{R}_\theta}
\end{equation*}

\subsubsection{TreeQN}
In this baseline, the starting step is encoding the input state to its latent counterpart with the Encoder module. Following this, a full-tree expansion, based on a predefined depth $d$, is performed using both Transition and Reward modules. The values at the leaf nodes are then backpropagated to the root node via the Bellman equation, as discussed in the Backup phase in \cref{section_method_tree_search_backup}. The root node Q-values serve as the final output, that is utilized for training. Given the exponential growth of TreeQN's computation graph with an increase in depth $d$, we choose a depth of 2 for both Procgen and navigation domains in our analysis, as used in TreeQN's original code base~\citep{farquhar2018treeqn}. Notably, greater depths, such as 3 or more, are infeasible since the resulting computation graph exceeds the memory capacity (roughly 11GB) of a typical consumer-grade GPU. The associated loss function, as adapted from the original paper, is:
\begin{equation*}
    \mathcal{L}_{TreeQN} = \lambda_1 \mathcal{L}_{Q} + \lambda_2 \mathcal{L}_{\mathcal{D}} + \lambda_3 \mathcal{L}_{\mathcal{R}_\theta}
\end{equation*}

It is important to note that every method is trained with same datasets using their respective loss functions. We fine-tune the hyperparameters, $\lambda_1,\lambda_2,\lambda_3 \text{ and } \lambda_4$, using grid search on a log scale.

\subsection{Issue with Baseline Normalized Score} \label{appendix_experiments_bns_issue}
Within the context of Atari games, the Baseline Normalized Score (BNS) is frequently utilized to evaluate the performance of agents. When human players are used as the baseline, it is often termed Human Normalized Score. The primary allure of BNS lies in its capacity to offer a relative assessment of an agent's performance, comparing it against a standard benchmark—this could be human players or even another agent.

One of the primary benefits of the BNS is its ability to provide a consistent metric across different games, addressing the difference in scale inherent in raw scores. By enabling the calculation of the average BNS across multiple games, we gain insight into the overall efficacy of an agent. This not only facilitates direct performance comparisons between diverse agents and methodologies but also paints a picture of how the agent's abilities stack up against human standards.

To derive the BNS, we start by logging the agent's raw score in an Atari game. This raw score is then normalized against a baseline score, derived from baseline agent's performance on the same game. By dividing the agent's score by the baseline's score (and sometimes subtracting the score of a random agent), we get a relative metric. Mathematically, this can be represented as:
\begin{equation}
    BNS(\pi) = \dfrac{S_\pi - S_R}{S_B - S_R}
\end{equation}
Here, $S_\pi, S_B$ and $S_R$ denote the raw scores of the agent, the baseline policy, and a random policy, respectively. Interpretation-wise, a BNS of 1 indicates parity with the baseline. Values exceeding 1 signify that the method outperforms, while those below 1 indicate that the method underperforms relative to the baseline.

Nevertheless, the BNS has its frailties. It inherently presumes the baseline policy will always surpass the performance of the random policy. But there can be instances, contingent on the environment or the specific baseline policy, where this isn't the case. In scenarios where the baseline policy underperforms the random policy, the BNS results in a negative denominator. This poses a predicament: even if our agent's policy performs better than the random policy, the BNS unfairly penalizes it. In our experiments with Procgen, we observed that for 2 out of the 16 games, namely heist and maze, the baseline policy underperformed compared to the random policy. Given these pitfalls, our evaluations pivot towards a more robust metric: the Z-score. The Z-score, often termed as the ``standard score", provides a statistical measurement that describes a value's relationship to the mean of a group of values. It is measured in terms of standard deviations from the mean. If a Z-score is 0, it indicates that the data point's score is identical to the mean score. Z-scores may be positive or negative, with a positive value indicating the score is above the mean and a negative score indicating it is below the mean.

\end{document}